\documentclass[lettersize,journal]{IEEEtran}
\usepackage{stfloats}
\usepackage[T1]{fontenc}
\usepackage{cite}
\usepackage{amsmath,amssymb,amsfonts}
\usepackage{algorithmic}
\usepackage{graphicx}
\usepackage{textcomp}
\usepackage{xcolor}
\usepackage{soul}
\usepackage{subfigure}
\usepackage{multirow}
\usepackage{algorithm}
\usepackage{amsthm}
\usepackage{amsmath}
\newtheorem{theorem}{Theorem}

\newtheorem{definition}{Definition}
\usepackage[cmintegrals]{newtxmath}
\usepackage{hyperref}
\hypersetup{hidelinks,
	colorlinks=true,
	allcolors=black,
	pdfstartview=Fit,
	breaklinks=true}

\usepackage{tikz,xcolor,hyperref}
\definecolor{lime}{HTML}{A6CE39}
\DeclareRobustCommand{\orcidicon}{
	\begin{tikzpicture}
		\draw[lime, fill=lime] (0,0) 
		circle [radius=0.16] 
		node[white] {{\fontfamily{qag}\selectfont \tiny ID}};    \draw[white, fill=white] (-0.0625,0.095) 
		circle [radius=0.007];    \end{tikzpicture}
	\hspace{-2mm}}
\foreach \x in {A, ..., Z}{
	\expandafter\xdef\csname orcid\x\endcsname{\noexpand\href{https://orcid.org/\csname orcidauthor\x\endcsname}{\noexpand\orcidicon}}
}

\begin{document}

\title{Meta-Computing Enhanced Federated Learning in IIoT: Satisfaction-Aware Incentive Scheme via DRL-Based Stackelberg Game}

\author{Xiaohuan Li, ~\IEEEmembership{Member,~IEEE,}
        Shaowen Qin, 
        Xin Tang,
        Jiawen Kang,
        Jin Ye, 
        Zhonghua Zhao,\\
        Yusi Zheng,
        and Dusit Niyato, ~\IEEEmembership{Fellow,~IEEE,}
    \thanks{This work was supported in part by the National Natural Science Foundation of China under Grant U22A2054, U23A20313, and in part by the Guangxi Natural Science Foundation of China under Grant 2025GXNSFAA069687. ({\itshape Corresponding author: Jin Ye, Zhonghua Zhao}.)}
    \thanks{Xiaohuan Li is with the Guangxi University Key Laboratory of Intelligent Networking and Scenario System (School of Information and Communication, Guilin University of Electronic Technology), Guilin 541004, China, and also with National Engineering Laboratory for Comprehensive Transportation Big Data Application Technology (Guangxi), Nanning 530001, China (e-mails: lxhguet@guet.edu.cn).}
    \thanks{Shaowen Qin, Xin Tang, Zhonghua Zhao and Yusi Zheng are with the Guangxi University Key Laboratory of Intelligent Networking and Scenario System (School of Information and Communication, Guilin University of Electronic Technology), Guilin 541004, China (e-mails: qinsw@mails.guet.edu.cn; tangx@mails.guet.edu.cn; yusizheng97@gmail.com; gietzzh@guet.edu.cn).}
    \thanks{Jiawen Kang is with the School of Automation, Guangdong University of Technology, Guangzhou 510006, China (e-mail: kavinkang@gdut.edu.cn).}
    \thanks{Jin Ye is with Guangxi Key Laboratory of Multimedia Communications and Network Technology, Nanning 530000, China, and also with School of Computer and Electronic Information, Guangxi University, Nanning 530000, China (e-mail: yejin@gxu.edu.cn). }
    \thanks{Dusit Niyato is with the College of Computing and Data Science, Nanyang Technological University, Singapore 639798 (e-mail: dniyato@ntu.edu.sg).}
}

\markboth{IEEE Transactions on Network and Service Management} 
{Shell \MakeLowercase{\textit{et al.}}: A Sample Article Using IEEEtran.cls for IEEE Journals} 

\maketitle

\begin{abstract}
The Industrial Internet of Things (IIoT) leverages Federated Learning (FL) for distributed model training while preserving data privacy, and meta-computing enhances FL by optimizing and integrating distributed computing resources, improving efficiency and scalability. Efficient IIoT operations require a trade-off between model quality and training latency. Consequently, a primary challenge of FL in IIoT is to optimize overall system performance by balancing model quality and training latency. This paper designs a satisfaction function that accounts for data size, Age of Information (AoI), and training latency for meta-computing. Additionally, the satisfaction function is incorporated into the utility function to incentivize IIoT nodes to participate in model training. We model the utility functions of servers and nodes as a two-stage Stackelberg game and employ a deep reinforcement learning approach to learn the Stackelberg equilibrium. This approach ensures balanced rewards and enhances the applicability of the incentive scheme for IIoT. Simulation results demonstrate that, under the same budget constraints, the proposed incentive scheme improves utility by at least 23.7\% compared to existing FL schemes without compromising model accuracy.
\end{abstract}
	

\IEEEpeerreviewmaketitle

\section{Introduction}
The Industrial Internet of Things (IIoT) serves as a core pillar of intelligent manufacturing and digital industrial systems, enabling real-time data collection, analysis, and intelligent decision making by connecting a large number of industrial devices, sensors, and intelligent systems \cite{c1,c2,c3}. Meta-computing \cite{b47} is an emerging computing paradigm that integrates and optimizes distributed computing resources in IIoT through key modules such as zero-trust computing management, task management, device management, identity and access management, and resource scheduling. These modules enable secure aggregation, dynamic task allocation, and efficient resource orchestration, thereby providing flexible and scalable computing services.
Federated Learning (FL) \cite{b8,b9} is a collaborative AI approach that enables IIoT to optimize global machine learning models without migrating local data from IIoT nodes, thus enhancing the intelligent analytics capabilities of meta-computing while ensuring data privacy \cite{b10,b50}.

To ensure efficient operation of IIoT, FL in meta-computing must meet key requirements such as efficient resource allocation, high-quality models, and low latency \cite{c6}. However, due to the large scale of IIoT devices and the complexity of data structures, FL in meta-computing still faces several challenges.

First, IIoT systems consist of numerous heterogeneous devices with varying computational capabilities and communication bandwidths, while task demands fluctuate dynamically with changes in the production environment. Some lowpower devices struggle with high-load computational tasks, whereas high-performance computing nodes may experience bottlenecks due to unbalanced resource allocation, reducing overall system efficiency \cite{c25}. Therefore, a key challenge is how to efficiently schedule tasks within limited resources to optimize resource allocation. Second, machine learning models are widely used in IIoT for intelligent prediction, anomaly detection, and automated control. However, high-quality models typically require significant computational resources and longer training times, while IIoT scenarios with large data volumes and complex structures often demand low latency and real-time responsiveness. Moreover, since IIoT nodes operate in distributed and autonomous environments, they are generally unwilling to share private information, further complicating the balance between model quality and latency \cite{c6}. Consequently, reducing computational and communication latency while maintaining model quality remains a critical research challenge.

The problem addressed in this paper is how to design an efficient incentive scheme that enables distributed IIoT devices with heterogeneous resources to collaboratively train high-quality models while minimizing communication and computation latency under limited resource and privacy constraints. We propose a novel evaluation metric, Satisfaction, to quantify the contribution of IIoT nodes in the FL process. This metric considers data volume, Age of Information (AoI) \cite{b40}, and computation latency to assess node value. Specifically, AoI reflects the freshness of model updates, prioritizing newer information, while computation latency evaluates node responsiveness, preventing high-latency devices from degrading overall training efficiency. Building on this metric, we develop a satisfaction-aware incentive mechanism that jointly formulates the utility functions of the server and participating nodes, dynamically allocates budgets, and incentivizes efficient resource sharing. This mechanism encourages nodes with high data quality and computational efficiency to actively participate in model training. Furthermore, the incentive mechanism is modeled as a Stackelberg Game, where IIoT nodes act as followers responding to the server’s leader decisions. Unlike traditional game-theoretic approaches, we employ Deep Reinforcement Learning (DRL) to learn optimal strategies through interaction experience, thereby eliminating the need for participating agents' private information. The main contributions of this paper are summarized as follows:

\begin{itemize}
\item We propose Meta-Computing Enhanced Federated Learning (MEFL) for dynamic computation scheduling in IIoT. MEFL leverages meta-computing resource scheduling to optimize FL computational resource allocation, employs task management to dynamically orchestrate FL training tasks, utilizes device management to coordinate distributed computing nodes, integrates zero-trust computing management to improve data privacy protection, and incorporates identity and access management for authentication and access control, ensuring secure resource access. MEFL empowers IIoT with efficient, secure, and adaptive computing capabilities, enabling intelligent decision-making and optimized resource utilization in distributed environments.

\item We introduce a new metric named “Satisfaction” to balance training latency and model quality. By considering data size, AoI, and computation latency, we measure the contribution of IIoT nodes and develop a Satisfaction-Aware incentive scheme, constructing utility functions for both nodes and the server to encourage resource sharing. 

\item We transform the optimization problem involving two utility functions into a Stackelberg game model and prove the uniqueness of the Stackelberg Equilibrium (SE). A DRL-based algorithm is employed to learn the SE without requiring private information from other agents, enabling optimal strategy learning purely from historical experience. 
\end{itemize}

The rest of the paper is organized as follows. Section~\ref{secII} reviews some important literature related to this paper. Section~\ref{secIII} gives the system model of the paper. Section ~\ref{secIV} details the satisfaction design. Section~\ref{secV} introduces the incentive scheme based on the satisfaction. Section~\ref{secVI} introduces the Stackelberg game based on DRL. Section~\ref{secVII} presents the experimental parameters and related results. The last section concludes the paper.

\section{Related Works}
\label{secII}
\subsection{FL in IIoT}
FL has been widely applied in IIoT to enable decentralized intelligence while preserving data privacy. However, the heterogeneous nature of IIoT devices poses significant challenges for efficient resource scheduling. The authors in \cite{c19} proposed a semi-decentralized federated edge learning (SD-FEEL) framework that optimizes device association, resource allocation, and edge server placement to minimize training loss and improve accuracy within a limited cost budget. The authors in \cite{c20} proposed a client scheduling scheme for a semi-decentralized FL (SD-FL) framework in IIoT, optimizing client-server association to improve communication reliability and training efficiency while minimizing global training loss. The authors in \cite{c21} proposed a FL-based Cyber Threat Intelligence framework (FL-CTIF) that improves IIoT security by using information fusion to improve cyberattack detection accuracy while reducing training rounds and CPU consumption. The authors in \cite{c22} proposed a decentralized federated learning framework for edge-enabled smart manufacturing, optimizing industrial data processing with inexact ADMM algorithms to improve low-latency and secure decision-making while improving response time and maintaining accuracy. The authors in \cite{c7} proposed a DT-empowered IIoT architecture optimized with FL, integrating deep reinforcement learning for device selection and an asynchronous FL scheme to address device heterogeneity. The authors in \cite{c8} proposed a framework FDEI integrating federated learning with DT-enabled IIoT to enhance service quality and trustworthiness. The authors in \cite{c9} proposed an energy-efficient federated learning framework for digital twin-enabled IIoT, where IIoT devices dynamically select training methods based on environmental conditions. The authors in \cite{c24} introduced a meta-computing driven vertical FL algorithm to address data incompleteness in heterogeneous IIoT devices, enabling collaborative intelligence without requiring shared feature spaces.

However, existing studies still struggle to efficiently schedule tasks in IIoT scenarios where heterogeneous device resources are limited and dynamically varying. The proposed MEFL framework leverages meta-computing for resource scheduling to optimize resource allocation and incorporates a satisfaction-aware incentive scheme, thereby enabling intelligent decision-making and improved resource utilization in IIoT environments.

\subsection{Incentive Schemes in IIoT} 
FL relies on distributed nodes to jointly participate in model training, and the performance of nodes directly affects the training latency and model quality \cite{c10}. Without an effective incentive scheme, selfish nodes might not provide sufficient resources or may be reluctant to join the federated learning task \cite{b20}. Therefore, designing an incentive scheme that fairly evaluates node contributions and encourages high-quality participation is crucial. The authors in \cite{c11} proposed the Model Value Transfer Incentive (MVTI) to enhance federated learning incentives for AIoT, ensuring secure execution and fair benefit redistribution with smart contracts and IPFS. The authors in \cite{c12} proposed BIMFL, a bilevel optimization-based incentive mechanism for federated learning, where the parameter server assigns rewards, and devices decide participation to maximize profit based on rewards and energy costs. The authors in \cite{c13} proposed a communication-efficient incentive mechanism for federated learning using Nash bargaining theory, formulating a concurrent bargaining game with a probabilistic greedy client selection algorithm. The authors in \cite{c18} proposed a multi-exit-based federated edge learning (ME-FEEL) framework for resource-constrained IIoT, enabling devices with limited computational power to exit early during training, reducing latency and maximizing participation in model aggregation. The authors in \cite{c23} proposed a social credit-based incentive mechanism to enhance cooperative D2D transmission across IIoT networks, improving throughput and reducing power consumption.

Although the above studies improve the fairness, participation, and communication efficiency of FL, they still struggle to simultaneously meet latency constraints and ensure model quality in dynamic IIoT scenarios. The proposed satisfaction-aware incentive scheme jointly considers data size, AoI, and service latency to quantify node contributions and balance model quality against service latency, thereby achieving an efficiency–accuracy trade-off in IIoT.

\subsection{Game Theoretic and DRL Approaches in IIoT}
Federated learning in IIoT involves complex resource allocation challenges. Traditional heuristic-based scheduling methods often fail to adapt to dynamic environments. Game-theoretic models have been explored to address resource competition in hierarchical decision-making scenarios, while DRL has emerged as a promising technique for learning optimal resource allocation policies in complex, dynamic IIoT networks. The authors in \cite{c14} proposed a blockchain-based resource trading framework for multi-UAV-assisted IIoT networks, leveraging Stackelberg game modeling and multi-agent DRL to optimize dynamic resource allocation while ensuring security and privacy. The authors in \cite{c15} proposed a DRL-based cloud-edge task scheduling model LsiA3CS, leveraging a Markov game and an asynchronous advantage actor-critic (A3C) algorithm to balance computing resources and reduce communication latency in large-scale IIoT networks. The authors in \cite{c16} focused on Stackelberg game-driven dynamic pricing and resource allocation, utilizing deep Q-networks (DQN) to approximate optimal strategies for resource pricing and task scheduling in IIoT.

However, existing approaches primarily focus on resource allocation and task offloading but do not explicitly incorporate Stackelberg game-theoretic modeling into FL for hierarchical optimization. Most solutions assume complete system information, which is impractical in privacy-sensitive FL scenarios where devices operate autonomously. Additionally, many DRL-based frameworks lack an incentive-aware mechanism to guide learning policies effectively. 

\section{System Model}
\label{secIII}
In order to promote efficient task processing and decision making in IIoT, we design MEFL, a meta-computing enhanced FL framework, as shown in Fig. \ref{fig1}. The device management module contains multiple edge nodes, whose primary purpose is to collect data from production devices, integrate the computational, storage, and communication resources of the edge nodes, and map these resources to the server. These resources are then converted into objects that can be easily accessed by the resource scheduler, enabling local FL training based on the incentive scheme developed by the task manager. The resource scheduler module contains several digital twins of edge nodes, which continuously monitor changes in the configuration details of physical nodes, simulate their possible states, and dynamically optimize resource allocation. The task management module, located on the server, accepts user requests, decomposes tasks, and designs incentive schemes based on conditional constraints. The zero-trust computing management module performs global aggregation for FL via blockchain, while the identity and access management module ensures users have appropriate permissions to access data. Since the zero-trust computing management and identity and access management modules are not the focus of this paper, detailed discussions are omitted. Comprehensive analyses of these modules are available in \cite{b47}. 

\begin{figure*}[!t]
	\centerline{\includegraphics[width=0.95\textwidth]{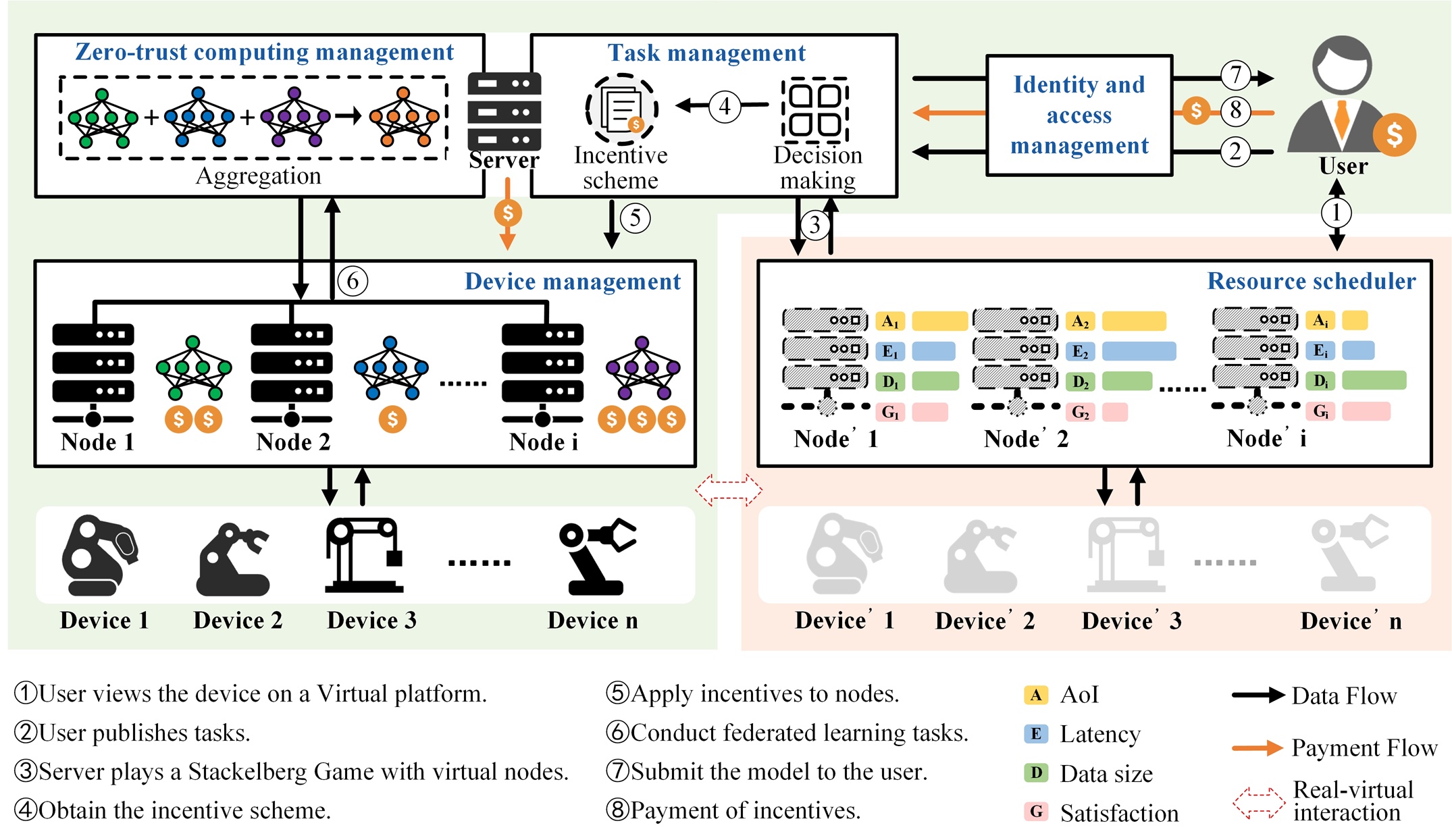}}
	\caption{The MEFL framework for IIoT. The framework enables resource scheduling and task management for IIoT devices, where the server employs a Satisfaction-aware incentive mechanism to coordinate nodes for efficient task execution.}
	\label{fig1}
\end{figure*}

In MEFL, users view devices on a virtual platform and publish tasks, which are transmitted to the task management module on server. The server interacts with virtual nodes in the resource scheduling module via a Stackelberg game and generates a satisfaction-aware incentive scheme based on the game results. The incentives are distributed to physical nodes in the device management module to drive local FL tasks. After completing model training, the physical nodes upload their updates to the zero-trust computing management module for aggregation. After multiple iterations, the final global model is produced and returned to the user, while the user completes the corresponding incentive payment.
The main notations used in this paper are shown in Table \ref{tab1}.

We represent the set of IIoT nodes involved as ${I_{all}} = \{1,...,i,...,I\}$ and the task duration as $T$. Upon task arrival, each model training is iterated $K$ times to minimize global losses, where $K$ is specified by the server. Assume that there are $I$ nodes with local data sets $\{D'_1,D'_2,...,D'_I\}$. We define $D_i\triangleq|D'_i|$, where $|\cdot|$ denotes the size of the data set with meta-computing, each node downloads a shared global model $\omega$ from the server and trains the model using its local data. The node then uploads the new local model update to the server. Therefore, the total size of data samples from $I$ nodes is $\sum\nolimits_{i=1}^{I}{{D}_{i}={D}_{all}}$. The loss function of the node $i$ with the data set $D'_i$ is
\begin{equation}
\begin{split}
{{F}_{i}}(\omega )\triangleq \frac{1}{{{D}_{i}}}\sum\limits_{s\in {{D}_{i}}'}{{{f}_{i}}(\omega )},\label{eq30} 
\end{split}
\end{equation}
where $f_i(\omega)$ is the loss function on the data sample $s$. The goal is to optimize the global loss function $F(\omega)$ by minimizing the weighted average of every node $i$’s local loss function $F_i(\omega)$ on its local training samples
\begin{equation}
\begin{split}
F(\omega )\triangleq \frac{\sum\nolimits_{i=1}^{I}{{{D}_{i}}{{F}_{i}}(\omega )}}{D_{all}},\label{eq31} 
\end{split}
\end{equation}
\begin{equation}
\begin{split}
\omega *=\arg \min F(\omega ).\label{eq32} 
\end{split}
\end{equation}

Meta-computing assists the training process and the exchange of models between nodes and the server. Specifically, it accelerates local update procedures and improves the efficiency of parameter transmission. As a result, each training round can complete sooner, and nodes can submit updates more quickly. This architectural support enhances the timeliness of federated learning and reduces resource overhead.

Due to the inherent complexity of many machine learning models, it is hard to find a closed-form solution. Therefore, it is often solved using gradient descent techniques \cite{b37}. Given that IIoT nodes may be reluctant to provide fresh sensing data for time-sensitive FL tasks, reliable incentives are greatly needed to encourage users to share fresh sensing data, which is discussed in the next section.

\begin{table}[t]
\caption{Summary of Main Notations}
\begin{center}
	\begin{tabular}{|p{1.2cm}|p{5.5cm}|}
		\hline
        \textbf{Notation}     &\textbf{Description}  \\ \hline
        $a_i$                 & The duration from the end of data collection to the beginning of the next period of data collection for node $i$         \\ \hline
		$c_i$                      &  The time spent by the node $i$ to collect and process the model training data         \\ \hline
		$d_i$               & The size of data collected per unit time period for node $i$        \\ \hline
        $r_i$                      & The satisfaction unit reward provided by the server for node $i$          \\ \hline
		$\sigma_i$                    & The unit cost required to maintain the update cycle for node $i$          \\ \hline
		$\theta _i$                    & The period of update buffer data for node $i$  \\ \hline
        $D_{i}$                      & The data size collected by node $i$          \\ \hline
        $E_i$                 &The service latency for node $i$            \\ \hline
		$G_i$                        & Server satisfaction with node $i$        \\ \hline
        $I_{all}$                      & The set of IIoT nodes involved          \\ \hline
        $M_i$                &The quality of the model provided by node $i$       \\ \hline
        $R_i$                 &The reward given to node $i$ by the server           \\ \hline
        $T$                 &The task duration            \\ \hline
        $U_i$                 &The utility of node $i$            \\ \hline
        $V_i$                 &The utility obtained by the server from node $i$           \\ \hline
	\end{tabular}
\label{tab1}
\end{center}
\end{table}

\section{Satisfaction: Quality Control for IIoT}
\label{secIV}
The structure of IIoT is complex, containing multiple types of industrial devices and involving a large number of industrial nodes, where the trade-off between low latency and high model quality must be carefully considered in meta-computing enhanced federated learning. Latency plays a critical role in industrial control and decision-making processes. High latency can lead to delayed data synchronization, causing inefficiencies in predictive maintenance, anomaly detection, and real-time monitoring \cite{c17}. Meanwhile, model quality determines the accuracy and reliability of decision-making. Degraded model quality means that the data used for decision-making is less accurate or outdated, which may lead to incorrect control actions, reduced production efficiency, or potential safety risks. Furthermore, model quality directly influences the effectiveness of IIoT decision-making processes—low-quality models may lead to suboptimal system adjustments or improper responses to changing industrial conditions. Traditional quality metrics in FL do not adequately capture the requirements of IIoT environments. Therefore, we propose a satisfaction metric $G_i$, to balance quality and latency. The satisfaction metric $G_i$ of the task $i$ is denoted as
\begin{equation}
\begin{split}
{G_{i}} = \tau {M_{i}} - \lambda{E_{i}},\label{eq10} 
\end{split}
\end{equation}
where $\tau>0$ and $\lambda>0$ are weighting parameters. $\tau$ weights the contribution of model quality to satisfaction, while $\lambda$ weights the penalty imposed by service latency.

Since the timeout of gradient updates can negatively impact the learning results, we use AoI to measure the freshness of the model in order to ensure its quality. AoI, as a valid measure of information freshness, denotes the latency of the information from its generation to the completion of the model training after it has been uploaded to the server \cite{b35}, and it can enhance the performance of time-critical applications and services. In FL, we assume that requests arrive at the beginning of each cycle. We focus on FL with a data cache buffer on the node and AoI \cite{b38}. The node $i$ periodically updates its cached data with an update period $\theta _i$ independently. It is denoted as
\begin{equation}
	{\theta _i} = {c_i}t + {a_i}t, \label{eq1}
\end{equation}
where ${c_i}t\ ({c_i} \in N)$ is the data collection and preprocessing time at node $i$ with meta-computing assistance, and ${a_i}t\ ({a_i} \in N)$ is the duration from the end of data collection to the beginning of the next phase of data collection, which includes the meta-computing assisted service period and any idle time.

When the request arrives during the data collection phase or at the beginning of phase $(c_i + 1)t$, the AoI is $t$. This is the minimum  AoI value. For requests arriving in phase $l \times t$, where $l \ge (c_i + 2)$, the AoI will be $[l - (c_i + 1) + 1]t$. We suppose that $t$ is fixed and that the update cycle ${\theta _i}$ is affected by ${c _i}$ and ${a _i}$. Let us consider a case with an adjustable update phase, i.e., when ${a_i} = a$ is fixed, ${c_i} = \frac{{{\theta _i}}}{t} - a$, we use ${\theta _i}$ instead of ${c _i}$. Therefore, the average AoI of the node $i$ is
\begin{equation}
\begin{split}
{A_i} &= \frac{t}{{{c_i} + {a_i}}}[{c_i} + 1 + \frac{{({a_i} - 1)({a_i} + 2)}}{2}]\\&
= \frac{{t{\theta _i}}}{{{\theta _i} - at}} + \frac{{{t^2}}}{{{\theta _i} - at}}(\frac{{{a^2} - a}}{2}),\label{eq3} 
\end{split}
\end{equation}
when ${\theta _i} > at$ , ${\bar A_i}({\theta _i})$ is a concave function about ${\theta _i}$.

In addition to AoI, traditional FL experiences service latency \cite{b39}, defined in this paper as ${E_{i}}$. Unlike AoI, service latency refers to the duration from when the node receives a request until it uploads a local model. It includes both the data collection period and the model training period.
The probability of a request arriving is uniformly distributed across periods, i.e., $\frac{1}{T}$. If the request arrives in the $n$th period of the data collection cycle, the service delay is ${c_i}t + t - (n - 1)t$. If the request arrives in any of the remaining periods, the service latency is $t$. Therefore, the average service latency is
\begin{equation}
\begin{split}
{E_i} &= \frac{{{c_i}}}{{{c_i} + {a_i}}}[({c_i}t + t) + \cdots + ({c_i}t + t - ({c_i} - 1)t)] + \frac{{{a_i}}}{{{c_i} + {a_i}}}t\\&
= \frac{{{c_i}}}{{{c_i} + {a_i}}}[\frac{{{c_i}}}{2}({c_i} + 3)] + \frac{{{a_i}}}{{{c_i} + {a_i}}}t\\&
= \frac{{{{({\theta _i} - at)}^3}}}{{2t{\theta _i}}} + \frac{{3{{({\theta _i} - at)}^2}}}{{2{\theta _i}}} + \frac{{a{t^2}}}{\theta }.
\label{eq4} 
\end{split}
\end{equation}

From Eq. \eqref{eq3} and \eqref{eq4}, we observe a trade-off between service delay and AoI when selecting the cycle length ${\theta _i}$. Intuitively, a lower ${\theta _i}$, which means that a shorter cycle length and more frequent data updates, results in a lower average AoI. However, service latency also increases as updates take time.

To ensure user satisfaction, we must maintain the freshness of the model while guaranteeing its quality. This requires that we cannot rely solely on raw data to evaluate node contributions. Therefore, we define the model quality contributed by each node as the ratio of data size to AoI, which is denoted as
\begin{equation}
\begin{split}
{M_i} = \rho \frac{{{D_i}}}{{{A_i}}} = \rho \frac{{Td({\theta _i} - at)}}{{{\theta _i}(t{\theta _i} + {t^2}(\frac{{{a^2} - a}}{2}))}} ,\label{eq9} 
\end{split}
\end{equation}
where $\rho>0$ is a quality scaling factor that transforms the AoI based ratio into a normalized quality measure, ensuring comparability with the latency penalty term.

In the context of combining the IIoT with FL, the final aggregated global model depends on the local models contributed by individual nodes, while the size of the raw data ${D_i}$ reflects the node's contribution to the overall FL training process. ${D_i}$ is denoted as ${D_i} = \frac{T}{{{\theta _i}}}d$, where $d$ is the amount of data collected per unit time period.

\section{Satisfaction-aware incentives scheme}
\label{secV}
\subsection{Utility Model and Problem Formulation}
To meet the quality and latency requirements of IIoT tasks, we formulate the corresponding utility functions and optimization objectives.
Node $i$ is incentivized to handle incoming task requests, and each participating node receives a monetary reward ${R_i}$ from the server. Therefore, the utility of node $i$ is the difference between the reward ${R_i}$ and the cost ${C_i}$ of participating in the FL task. The utility $U_i$ can be expressed as
\begin{equation}
\begin{split}
{U_i} = {R_i} - {C_i},\label{eq5} 
\end{split}
\end{equation}
where the cost of the FL training task is defined as
${C_i} = \frac{{{\sigma _i}}}{{{\theta _i}}}$, where ${\sigma _i}$ is the unit cost required to maintain the update cycle, ${\theta _i}$ with respect to data collection, computation and transmission. 

Additionally, to incentivize nodes to participate in FL and provide higher-quality local models, the server will provide corresponding rewards ${R_i}$, which is denoted as ${R_i} = {r_i}\ln (\frac{1}{{{\theta _i}}})$, where ${r_i}$ is a unit award. The utility that the server receives from node $i$ is defined as the difference between the satisfaction gain $\beta {G_i}$ obtained by the server and the payoff ${R_i}$ given to the node, expressed as
\begin{equation}
\begin{split}
{V} =\sum\limits_{i\in {{I}_{all}}} (\beta {G_i} - {R_i}),\label{eq12} 
\end{split}
\end{equation}
where $\beta $ is the profit per unit of satisfaction.

In order to meet the quality and latency requirements of IIoT tasks, we define the goal of the incentive mechanism as maximizing server utility function $V$ and node utility function $U_i$.
The optimization problem for the node utility  is formulated as follows:
\begin{equation}
\begin{split}
  & \text{P1: }\underset{{{\theta }_{i}}}{\mathop{\max }}\, {{U}_{i}} \\ 
 & s.t. {{\theta }_{i}}\ge \theta _{i}^{\min },
\label{eq41} 
\end{split}
\end{equation}
which is subject to a minimum update cycle $\theta_i^{\max}$.
The optimization problem for the server utility is formulated as follows:
\begin{equation}
\begin{split}
  & \text{P2: }\underset{{{r}_{i}}}{\mathop{\max }}\,{{V}_{}} \\ 
 & s.t.{{A}_{i}}\le A_{i}^{\max }, \\ 
 & {{D}_{i}}\le D_{i}^{\max }, \\ 
 & \sum\limits_{i\in I}{{{r}_{i}}}\le {{R}^{\max }},  
\label{eq42} 
\end{split}
\end{equation}
which is subject to a maximum tolerable AoI constraint ${A_i}^{\max}$ , the maximum tolerable service delay constraint ${D_i}^{\max}$ , and a budget constraint ${R^{\max}}$.

\subsection{Stackelberg Game Analysis}
Since the goal of the server and the node is to maximize their respective reward functions Eq. \eqref{eq41} and \eqref{eq42}, We model the interaction between the server and the node as a two-part Stackelberg game, where the server, as a leader, determines the reward strategy ${r_i}$ and the node, as a follower, responds with ${\theta _i}$. The Stackelberg game can be defined in terms of strategies as
\begin{equation}
\begin{split}
\Omega  = \{ (SP \cup {\{ i\} _{i \in I}}),(r_i,\theta _i),(V_i,U_i)\} .\label{eq14} 
\end{split}
\end{equation}

The policy consists of three parts, $(SP \cup {\{ i\} _{i \in I}})$ denotes the set of servers and corresponding nodes, $({r _i},{\theta _i})$ denotes the set of policies, and $({V_i},{U _i})$ denotes the set of utilities.
We denote ${\theta ^*}$ as the optimal update period provided by the node, i.e., ${\theta ^*} = [\theta _1^*,...,\theta _i^*,...,\theta _I^*] $, and ${r_i^*}$ as the optimal reward decision of the server.
\begin{definition}
(Stackelberg Equilibrium): There exists an optimal update cycle $\theta _i^*$ and an optimal reward $r_i^*$. A policy $(\theta _i^*, r_i^*)$ is considered a Stackelberg equilibrium if and only if it satisfies the following 
\begin{equation}
\begin{split}
\forall \theta _i^{},U_i^{}(\theta _i^*,r_{i}^*) \ge U_i^{}(\theta _i^{},r_{i}^*)\\
\forall r_i^{},V(\theta _i^*,r_{i}^*) \ge U_i(\theta _i^*,r_i).\label{eq15}
\end{split}
\end{equation}
\end{definition}

The node's optimal decision is analyzed using the standard inverse induction method. Subsequently, we compute the first-order and second-order derivatives of ${U _i}$ with respect to ${\theta _i}$ as follows:
\begin{equation}
\begin{split}
  & \frac{\partial U_{i}^{{}}}{\partial \theta _{i}^{{}}}=\frac{{{\sigma }_{i}}}{\theta _{i}^{2}}-\frac{{{r}_{i}}}{{{\theta }_{i}}} \\ 
 & \frac{{{\partial }^{2}}U_{i}^{{}}}{\partial \theta _{i}^{2}}=\frac{{{r}_{i}}{{\theta }_{i}}-2{{\sigma }_{i}}}{\theta _{i}^{3}},  \label{eq16}
\end{split}
\end{equation}
where ${{\theta }_{i}}<\frac{2{{\sigma }_{i}}}{{{r}_{i}}}$, $\frac{{{\partial }^{2}}U_{i}^{{}}}{\partial \theta _{i}^{2}}<0$, and $U_{i}^{j}$ are convex functions. Solving the first-order derivative optimality condition$\frac{\partial U_{i}^{{}}}{\partial \theta _{i}^{{}}}=0$ is obtained as $\theta _{i}^{*}=\frac{{{\sigma }_{i}}}{{{r}_{i}}}$. By setting ${{\theta }_{i}}={{\theta }_{i,\max }}$ and ${{\theta }_{i}}={{\theta }_{i,\min }}$, we can get the upper and lower bounds of $\theta$ with respect to each node $i$. Based on this, the optimal update response is obtained as
\begin{equation}
\theta _i^* = \left\{ \begin{array}{l}
{\theta _{i,\max }},{\rm{                        }}{\theta _i} > {\theta _{i,\max }},\\
\frac{{{\sigma _i}}}{{{r_i}}},{\rm{                          }}{\theta _{i,\min }} < {\theta _i} < {\theta _{i,\max }},\\
{\theta _{i,\min }},{\rm{                        }}{\theta _i} < {\theta _{i,\min }}.
\end{array} \right.\label{eq18}
\end{equation}

By substituting $\theta _{i}^{*}$ into Eq.(\ref{eq12}),the server's utility $V$ can be re-expressed as
\begin{equation}
\begin{array}{ll}
V=&\sum\limits_{i \in I} (\beta \tau \rho \frac{{Td({\sigma _i} - atr)}}{{{\sigma _i}(t\frac{{{\sigma _i}}}{r} + {t^2}(\frac{{{a^2} - a}}{2}))}} \\&- 
\beta \lambda \frac{{({{(\frac{{{\sigma _i}}}{r} - at)}^3} + 3t{{(\frac{{{\sigma _i}}}{r} - at)}^2} + 2a{t^3})r}}{{2t{\sigma _i}}} - r\ln (\frac{r}{{{\sigma _i}}})).
\end{array} 
\label{eq19}
\end{equation}

The server utility optimization problem with feasibility constraints is reformulated as
\begin{equation}
\begin{split}
\underset{r}{\mathop{\max }}\,&V\\
 s.t. &\max \{A_{i}^{{}}\}\le {{A}_{i}}^{\max }, \\ 
 & \max \{D_{i}^{{}}\}\le {{D}_{i}}^{\max }.
 \label{eq20}
\end{split}
\end{equation}

The first-order and second-order derivatives of ${{V}_{i}}$ with respect to $r_i$ are obtained as follows:
\begin{equation}
\begin{split}
\begin{array}{l}
\frac{{\partial {V_i}}}{{\partial {r_i}}} =  - \frac{{2T\rho d\tau \beta  \cdot \left( {\left( {{a^3} - {a^2}} \right){t^2}{r^2} + 4a\sigma tr - 2{\sigma ^2}} \right)}}{{\sigma t \cdot {{\left( {\left( {{a^2} - a} \right)tr + 2\sigma } \right)}^2}}}\\
\qquad \; \:{\rm} + \frac{{\lambda \beta  \cdot \left( {\left( {{a^3} - 3{a^2} - 2a} \right){t^3}{r^3} + \left( {3 - 3a} \right){\sigma ^2}tr + 2{\sigma ^3}} \right)}}{{2\sigma t{r^3}}} - \ln \left( {\frac{r}{\sigma }} \right) - 1,\\
\frac{{{\partial ^2}{V_i}}}{{\partial {r_i}^2}} =  - \frac{{8Ta \cdot \left( {a + 1} \right)\rho d\tau \sigma \beta }}{{{{\left( {\left( {{a^2} - a} \right)tr + 2\sigma } \right)}^3}}} - \frac{{3\lambda \sigma \beta \Lambda }}{{t{r^4}}} - \frac{1}{r},
\end{array}
\label{eq21}
\end{split}
\end{equation}
where $a>1$, ${{\theta }_{i}}>at$, $\theta _{i}^{*}=\frac{{{\sigma }_{i}}}{{{r}_{i}}}>at$, and then $r<\frac{{{\sigma }_{i}}}{at}<\frac{{{\sigma }_{i}}}{(a-1)t}$, we obtain $\Lambda =\sigma -\left( a-1 \right)tr>0$. The second order derivative $\frac{{{\partial }^{2}}V_{i}^{{}}}{\partial {{r}_{i}}^{2}}<0$ is finally obtained. It shows that the objective function in the problem is convex with respect to ${{r}_{i}}$ and the problem in Eq. \eqref{eq19} is a convex optimization problem, which can be solved by existing convex optimization tools to find the optimal solution of ${{r}_i^{*}}$ .

\begin{theorem}
There exists a unique stackelberg equilibrium in the proposed game $(\theta _{i}^{*},r_{i}^{*})$.
\end{theorem}

\begin{proof}
Based on the given reward policy $r$, each node has an optimal update policy $\theta _{i}^{*}$, which is unique due to the convex characterization of the payoff function ($\frac{{{\partial }^{2}}U_{i}^{{}}}{\partial \theta _{i}^{2}}<0$). Next, the server has a unique optimal policy under the best response of all nodes ($\frac{{{\partial }^{2}}V_{i}^{{}}}{\partial {{r}_{i}}^{2}}<0$). The unique equilibrium is reached since the final policy $(\theta _{i}^{*},r_{{}}^{*})$ maximizes the node's utility and the server's utility, respectively. $\qedsymbol$
\end{proof}

As each IIoT node is modeled as a rational agent seeking to maximize its own utility, it is unlikely to intentionally harm the system, since such behavior would ultimately reduce its expected reward under the Stackelberg equilibrium. Even when nodes attempt limited strategic misreporting, the server can still infer their true state through observable indicators such as AoI, service latency, and participation frequency. Consequently, the proposed scheme demonstrates robustness against unreliable nodes and mild strategic behaviors. 

\section{DRL-based Stackelberg Game}
\label{secVI}
\subsection{DRL for Stackelberg Game}
Traditional heuristic algorithms require full information about the game environment. However, due to the non-cooperative nature of the game, players are typically unwilling to disclose their private information. DRL aims to learn decision-making based on past experiences, current states, and given rewards. To address decision-making problems with continuous action spaces, this paper employs the multi-agent deep deterministic policy gradient (MADDPG) algorithm \cite{b55}, a DRL algorithm designed for multi-agent environments. We describe the detailed definition of each term as follows.

\textbf{State Space:} In the current decision round $t$, the state space is defined by the price strategy $R^t = \{r^t_1, \dots, r^t_i, \dots, r^t_I\}$ assigned by the server to each edge node, and the caching strategy $\Theta^t = \{\theta^t_1, \dots, \theta^t_i, \dots, \theta^t_I\}$ of the edge node. Formally, the state space at round $t$ is represented as $S^t \overset{\Delta}{=} \{R^t, \Theta^t\}$.

\textbf{Partially Observable Space:} For privacy protection reasons, the agents at the edge nodes cannot observe the complete state of the environment and can only make decisions based on their localized observations in the formulated partially observable space. At the beginning of each training round $t$, the server first decides on its strategy based on its own historical strategy, which can be regarded as the observation space of the server: $o^t_{server} \overset{\Delta}{=} \{R^{t-L}, \Theta^{t-L}, \dots, R^{t-1}, \Theta^{t-1}\}$. Then, edge node $i$ determines its caching strategy based on the historical pricing strategy of the server and the historical update strategies $\Theta^t_{-i} = \{\theta^t_1, \dots, \theta^t_{i-1}, \theta^t_{i+1}, \dots, \theta^t_I\}$ of the other edge nodes. Therefore, the observation space of edge node $i$ is denoted as $o^t_{node} \overset{\Delta}{=} \{R^{t-L}_i, \Theta^{t-L}_{-i}, \dots, R^{t-1}_i, \Theta^{t-1}_{-i}\}$.

\textbf{Action Space:} After receiving the observation $o^t_{server}$, the server agent must take an action $x^t_{server} = r^t$ with the goal of utility maximization. Considering the bid limit $r^{max}$, the action space is defined as $r\in[0, r^{max}]$. The edge node determines a caching strategy $x^t_{node} = \theta^t$ after receiving the observation $o^t_{node}$.

\textbf{Reward Function:} After all agents take actions, each agent receives an immediate reward $e_t$ corresponding to the current state and the action taken. The reward functions of the nodes and server are aligned with the utility functions in Eq. \eqref{eq5} and \eqref{eq12}.

In each training cycle, the server determines a payment strategy. Upon observing the payment strategy from the server, the edge nodes determine their feedback. After the server receives the optimal training strategies from the edge nodes, it proceeds to determine the payment strategy. The server also updates the strategy and value function based on the rewards from each training cycle.

\subsection{MADDPG Algorithm}
\begin{figure}[t!]
	\centerline{\includegraphics[width=0.49\textwidth]{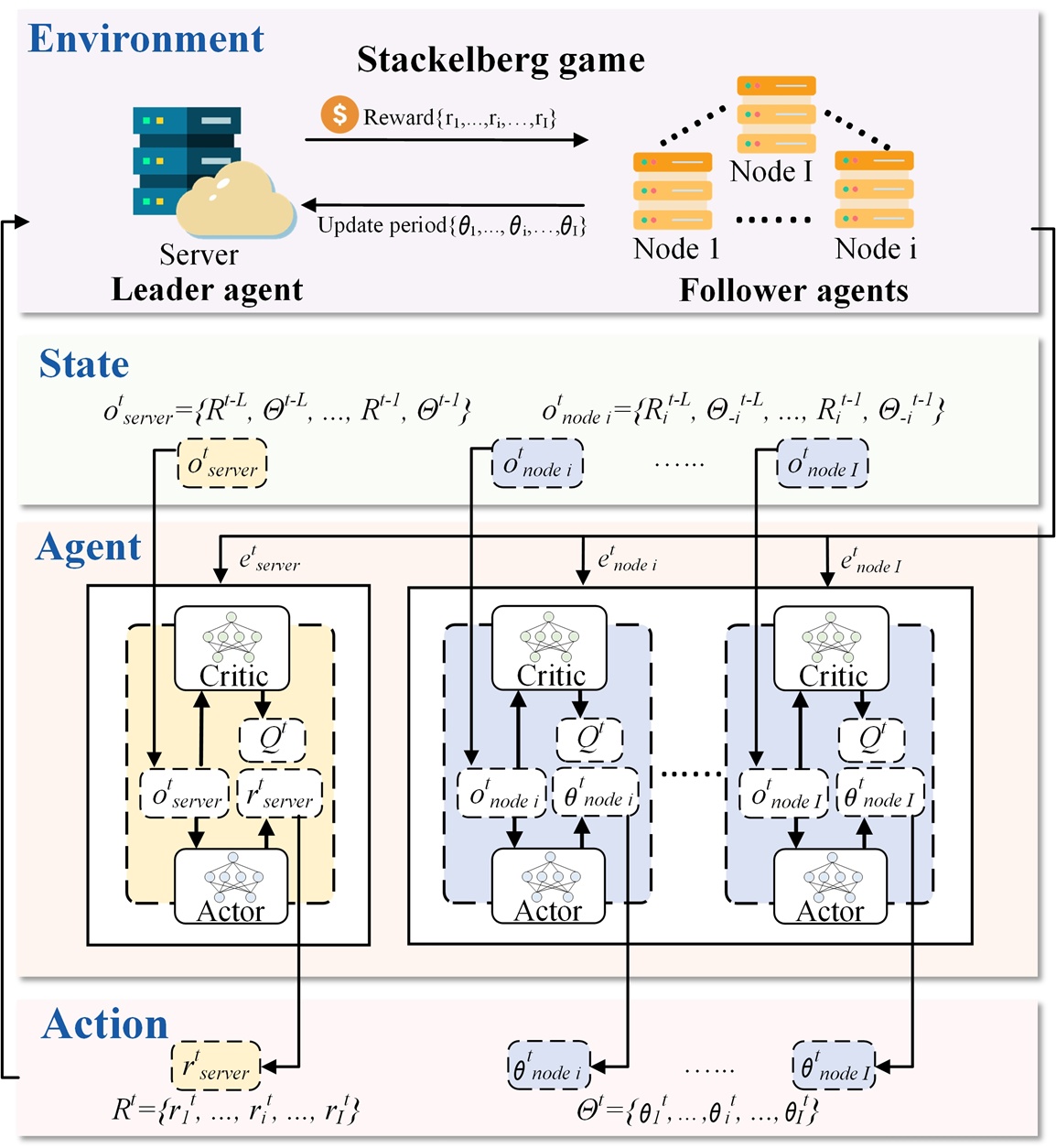}}
        \caption{DRL algorithm for Stackelberg game.}
	\label{fig11a}
\end{figure}

\begin{figure}[b]
	\centerline{\includegraphics[width=0.5\textwidth]{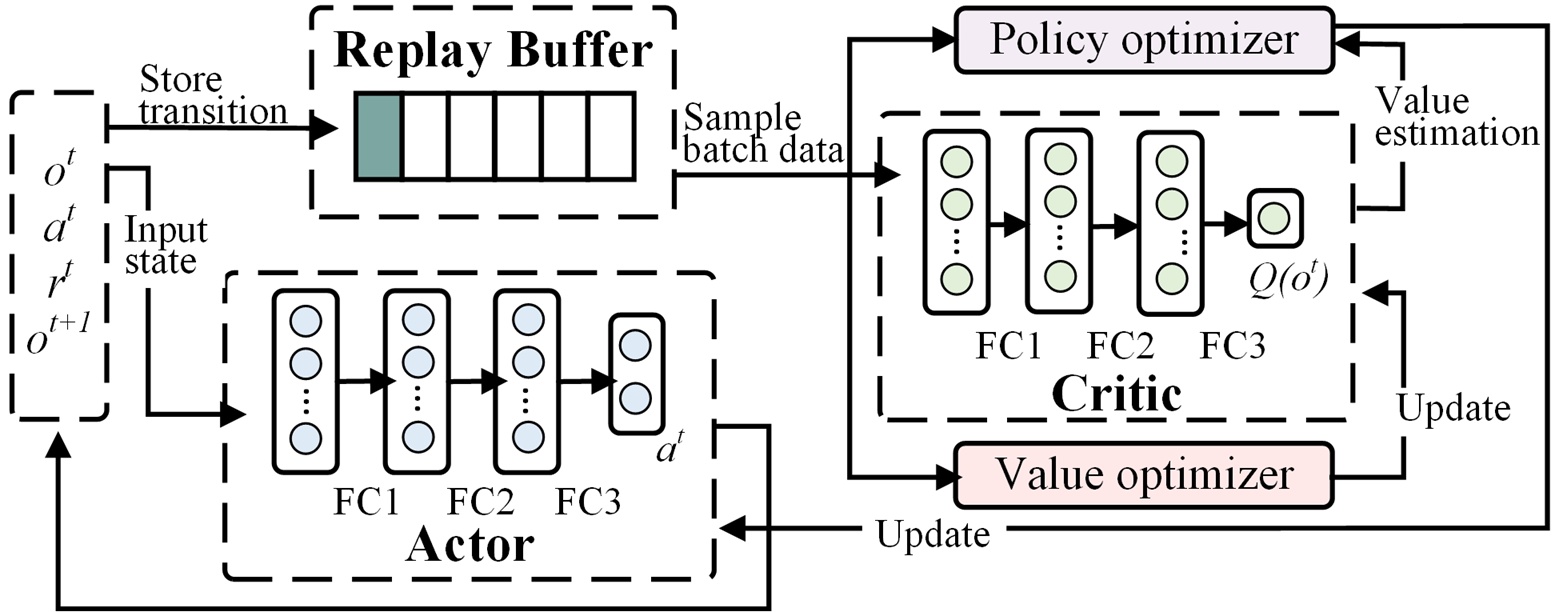}}
        \caption{Details of the DRL Controller.}
	\label{fig11b}
\end{figure}

The process of the MADDPG-based Stackelberg game is illustrated in Fig. \ref{fig11a}, where the server serves as the leader and nodes act as the followers. In each training cycle, the server agent observes the state $o^t_{server}$ and determines the action $x^t_{server}$, while the node agent observes the state $o^t_{node}$ and determines the action $\theta^t_{node}$. Subsequently, the current state transitions to the next, and the agents receives the reward. The detailed components of the DRL controller for each agent are shown in Fig. \ref{fig11b}. The replay buffer is used to store transitions, including the current state, action, reward, and next state, collected during interactions with the environment. These stored transitions are sampled in batches to decorrelate sequential data and stabilize the training process. The actor and critic network comprising three fully connected layers. The actor network takes the current state as input and outputs the corresponding action by generating a policy. The critic network evaluates the action taken by the actor network and provides a value estimation to guide policy improvement. Both the actor and critic networks are updated using two separate optimization modules. The policy optimizer updates the parameters of the actor network based on the policy gradient, while the value optimizer minimizes the temporal difference error to refine the critic network's value estimations.
\begin{algorithm}[!b]
\caption{MADDPG-based Stackelberg Game}\label{alg1}
\begin{algorithmic}[1]
    \STATE Initialize maximum episodes $E$, maximum time steps $T$ in an episode, batch size $B$
    \FOR{Each agent}
        \STATE Initialize critic network $Q(s_t,x_t|\phi^Q$) and actor $\mu(s_t|\phi^\mu)$ with weights $\phi^Q$ and $\phi^\mu$
    \ENDFOR    
    \WHILE{$e\le E$}
        \STATE Initialize a random process $N$ for action exploration
        \STATE Receive initial observation state $o_1$
        \WHILE{$t\le T$}
            \STATE Select action $x_t=\ \mu(o_t|\phi^\mu)+N_t$ according to the current policy
            \STATE Execute action $x_t$, and observe reward $e_t$ and observe new state $o_{t+1}$
            \STATE Store transition $(o_t,x_t,e_t,o_{t+1})$ in $B$
            \STATE Sample a random minibatch of $N$ transitions $(o_i,x_i,e_i,o_{i+1})$ from $B$
            \STATE Set $y_i$ by Eq. \eqref{eq106}
            \STATE Update critic by minimizing the loss formula Eq. \eqref{eq105}
            \STATE Update the actor policy using the sampled policy gradient formula Eq. \eqref{eq104}
            \STATE Update the target networks by Eq. \eqref{eq107}
            \STATE $t=t+1$
        \ENDWHILE
        \STATE $e=e+1$
    \ENDWHILE       
\end{algorithmic}
\end{algorithm}

The actor network is responsible for outputting a deterministic action $x_t$  based on the current environment state $s_t$. The actor network $\mu(o_t|\phi^\mu)$ is defined by the parameters $\phi^\mu$. The parameters $\phi^\mu$ of the actor network are updated using the policy gradient formula, given as
\begin{equation}
\begin{split}
\nabla_{\phi^\mu}J\approx \mathbb{E}_{o_t\sim\rho^\beta}\left[\nabla_{\phi^\mu}Q\left(o_t,\mu\left(o_t\middle|\phi^\mu\right)\middle|\phi^Q\right)\right],\label{eq104}
\end{split}
\end{equation}
where $\rho^\beta$ denotes the state distribution under the behavior policy $\beta$, while $\nabla_{\phi^\mu}Q\left(o_t,\mu\left(o_t\middle|\phi^\mu\right)\middle|\phi^Q\right)$ represents the gradient of the value function for the action chosen by the actor network in the state $o_t$.

The critic network $Q(o_t,x_t|\phi^Q)$ is defined by the parameters $\phi^Q$. It is trained by minimizing the mean squared error between the predicted Q-values and the target Q-values, with its loss function $L$ defined as
\begin{equation}
\begin{split}
L\left(\phi^Q\right)=\mathbb{E}_{\left(o_t,x_t,e_t,o_{t+1}\right)\sim\mathbb{D}}\left[\left(y_t-Q\left(o_t,x_t\middle|\phi^Q\right)\right)^2\right],\label{eq105}
\end{split}
\end{equation}
where $\mathbb{D}$ represents the experience replay buffer, and $y_t$ is the target Q-value.

The DDPG employs Temporal Difference (TD) learning to update the critic network, with the target value $y_t$ calculated based on the Bellman equation as
\begin{equation}
\begin{split}
y_t=e_t+\gamma Q\left(o_{t+1},\mu\left(o_{t+1}\middle|\phi^{\mu^\prime}\right)\middle|\phi^{Q^\prime}\right),
 \label{eq106}
\end{split}
\end{equation}
where $e_t$ is the immediate reward, $\gamma$ is the discount factor, and $\phi^{\mu^\prime}$ and $\phi^{Q^\prime}$ represent the parameters of the target actor network and target critic network, respectively. These target network parameters are softly updated at a slow rate $\tau$ towards the main network parameters $\phi^\mu$ and $\phi^Q$ to ensure stability during the training process
\begin{equation}
\begin{split}
\phi\gets\tau\phi+(1-\tau)\phi^\prime.
 \label{eq107}
\end{split}
\end{equation}

By alternately updating the actor and critic networks using the above loss function and gradient update formulas, the policy gradually converges to the optimal strategy that maximizes the cumulative reward.

The MADDPG-based solution for the Stackelberg game is presented in Algorithm~\ref{alg1}. In each training step, the MADDPG algorithm samples and stores experiences from $I$ agents and then samples a batch of $B$ experiences from the stored memories to train each agent. The MADDPG algorithm utilizes a deep neural network to construct both an actor network and a critic network, each comprising an input layer, a hidden layer, and an output layer. The primary factor influencing the time complexity is the dimensionality of the network architecture. Let $S$ represent the dimensionality of the state space, $L$ represent the dimensionality of the action space, and $H$ represent the number of neurons in the hidden layer. The computational complexity of the actor network is $O(SH + HL + H^2)$, while the complexity of the critic network is $O((S + L)H + H^2 + H)$. Since the target actor and critic networks have the same architecture, the complexity for a single agent is $O(2(SH + HL + H^2) + 2((S + L)H + H^2 + H))$. As a result, the overall time complexity of the algorithm is $O(2IB(2H^2 + 2(S + L)H + H))$.

In addition, as the number of participating nodes $I$ increases, the parallel structure of the actor–critic networks allow each agent to update its policy independently, while the centralized critic aggregates global feedback with a linear communication cost $O(I)$. This decentralized training and centralized evaluation architecture ensures that the system’s overall computational growth remains approximately linear rather than exponential. Regarding time varying network conditions, the continuous online update of the server’s strategy and each node’s update cycle enables adaptive re-optimization without retraining from scratch. The DRL agents learn dynamic mappings from observed states to actions, allowing the system to automatically adjust to variations in communication delays, node failures, or changing participation frequencies. Consequently, the proposed method maintains convergence stability and utility optimization even when applied to real world IIoT deployments with hundreds of heterogeneous nodes and dynamically evolving conditions.

\begin{table}[!t]
\caption{Training Parameters Information}
\begin{center}
	\begin{tabular}{|p{5cm}|p{1.5cm}|}
		\hline
		\textbf{Parameter}                      & \textbf{Set up}         \\ \hline
		Total number of nodes $I$                 & [5,25]          \\ \hline
		Duration from the end of data collection to the beginning of the next phase of data collection $a$                & [1,8]          \\ \hline
		Duration of training missions $T$                    & 10          \\ \hline
		Total model training time period $t$                    & 1  \\ \hline
		Update data volume $d$                        & [10,80]          \\ \hline
		Parameters of model quality $\rho $                        & [3,7]        \\ \hline
        Satisfaction Unit Profit $\beta $                &3       \\ \hline
	\end{tabular}
\label{tab2}
\end{center}
\end{table}

\section{Simulation Results}
\label{secVII}
To verify the impact of the proposed satisfaction-aware incentive scheme on FL in the IIoT, this paper conducts experiments using Python 3.7 and TensorFlow 1.15, and evaluates the performance of the scheme. The value ranges of the simulation parameters are shown in Table \ref{tab2}. We take image classification tasks as an example of IIoT applications. In the physical space, the robots perform image classification tasks without loss of generality. We use the MNIST dataset, consisting of 60,000 samples for the training set and 10,000 samples for the test set. To evaluate the performance of the proposed scheme, we compare it with four reinforcement learning algorithms: multi-agent proximal policy optimization‌ (MAPPO) \cite{b53}, multi-agent soft actor-critic (MASAC) \cite{b54}, and multi-agent deep Q-networks‌ (MADQN) \cite{b56}, which are PPO algorithm, SAC algorithm, and DQN algorithm in multi-agent environment.

\subsection{Satisfaction and Utility Analysis}
Fig. \ref{fig2} illustrates the impact of update cycles on satisfaction. In the early stages of the update cycle, the AoI decreases rapidly with an increase in the update cycle, significantly improving model quality and leading to a rapid increase in satisfaction. However, as the update cycle reaches a certain threshold, the impact of increasing latency becomes apparent, causing satisfaction to peak and then decline. This indicates the existence of an optimal update cycle for optimizing node satisfaction. Moreover, satisfaction decreases with an increase in parameter $a$ and increases with an increase in parameter $d$.

Fig. \ref{fig3} shows a comparative analysis of server utility under different incentive schemes with the maximum budget constraint ($R^{max}=50$). The utility of our scheme reaches its peak when 15 nodes are chosen, indicating an optimal balance between the number of nodes and the budget constraint. While increasing the number of nodes usually results in higher utilities, due to the budget limit, utilities decrease after the optimal point. 
However, this does not mean that our scheme cannot handle larger scale networks. The peak point corresponds to the equilibrium where the marginal utility of adding new participants is balanced by the fixed incentive budget. As more nodes join, the average reward per node and the available computational resources become diluted, leading to a gradual saturation of performance rather than degradation. Even as the network scales up, the proposed scheme maintains stable utility and balanced resource consumption under the given budget constraints.
The results suggest that our scheme consistently outperforms other strategies across all node numbers, indicating an optimized balance between cost and quality incentives to maximize utility within the given budget. The quality-first and price-first strategies show competitive performance, although neither dominates across all node numbers. The randomized pricing and randomized node combination strategies yield lower utility, which may suggest a less effective node selection or pricing model under these approaches.

\begin{figure}[!t]
	\centerline{\includegraphics[width=0.5\textwidth]{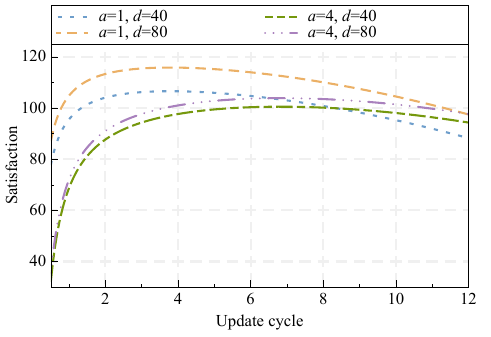}}
        \caption{Satisfaction for different update cycle.}
	\label{fig2}
\end{figure}

\begin{figure}[!t]
	\centerline{\includegraphics[width=0.5\textwidth]{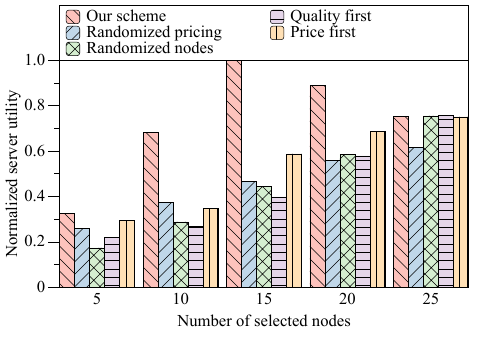}}
        \caption{Server utility for different schemes and numbers of selected nodes.}
	\label{fig3}
\end{figure}

\begin{figure*}[!t]
	\centering
    \begin{minipage}{0.8\textwidth}
        \centering
        \includegraphics[width=0.35\textwidth]{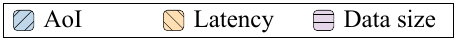}
    \end{minipage}
	\subfigure[$\tau=9, \lambda=1$]{
		\includegraphics[width=0.3\textwidth]{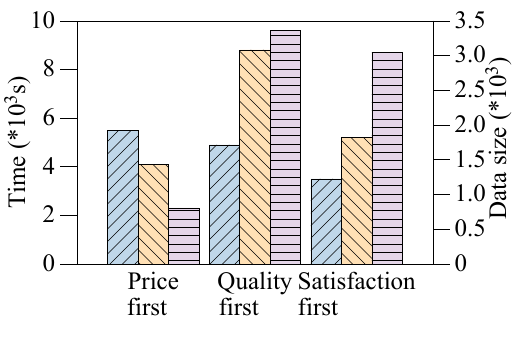}\label{fig10a}
	}
        \subfigure[$\tau=5, \lambda=5$]{
		\includegraphics[width=0.3\textwidth]{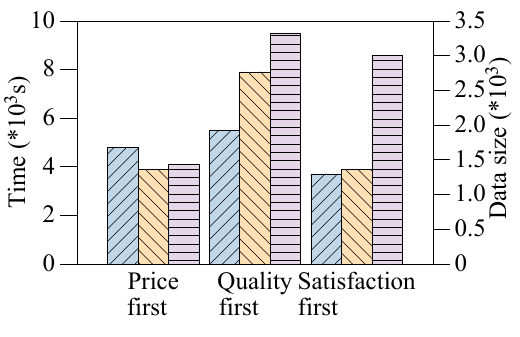}\label{fig10b}
	}
        \subfigure[$\tau=1, \lambda=9$]{
		\includegraphics[width=0.3\textwidth]{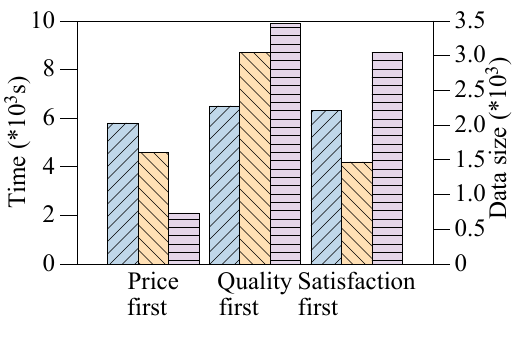}\label{fig10c}
	}
        \caption{Scheme performance under different preferences.}
	\label{fig10}
\end{figure*}

\begin{figure}[!t]
	\centerline{\includegraphics[width=0.5\textwidth]{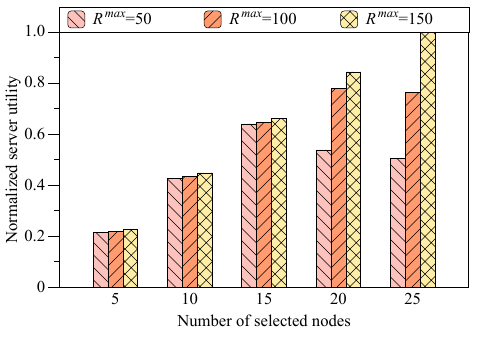}}
        \caption{Server utility for different number of selected nodes and budget constraints.}
	\label{fig5}
\end{figure}

\begin{figure}[!t]
	\centering
	\subfigure[]{
		\includegraphics[width=0.45\textwidth]{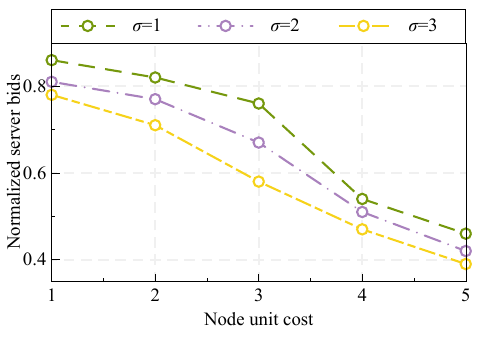}\label{fig6a}
	}
        \subfigure[]{
		\includegraphics[width=0.45\textwidth]{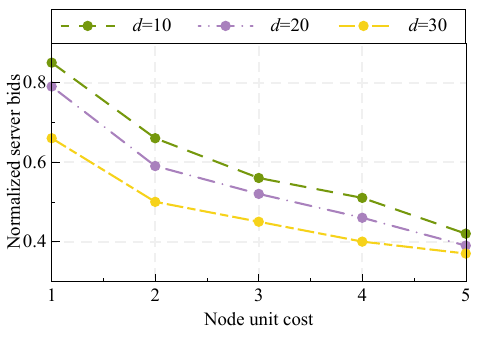}\label{fig6b}
	}
        \caption{Normalized server strategies for different parameters and unit cost of node.}
	\label{fig6}
\end{figure}

Fig. \ref{fig10} illustrates the performance comparison of different schemes under different preferences. Fig. \ref{fig10a} illustrates a situation where the user’s preference for quality far outweighs the concern for latency. With this preference, the satisfactionfirst strategy minimizes the value of AoI, indicating that the information is updated very frequently and that it is likely to perform well in providing quality services. At the same time, the price first strategy maintains the smallest data size, which indicates that the economy of data transmission is achieved even when quality is the primary concern. Fig. \ref{fig10b} illustrates the case where the user’s preference for both quality and latency is not strong. In this setting, the strategies are close in terms of AoI and latency, but the price first strategy performs the smallest in terms of data size, which is consistent with the observation in Fig. \ref{fig10a}. Fig. \ref{fig10c} shows the setting where users prefer low latency. The price first strategy performs minimally in terms of data size, which may imply that fewer data are transmitted in the price first strategy in order to reduce the cost. The quality first strategy has a lower AoI, indicating that it provides more frequent information updates.

Fig. \ref{fig5} illustrates the comparison of normalized server utility under different budget constraints. The figure shows three bar sets corresponding to different maximum budget constraints: $R^{max}=50$, $R^{max}=100$, and $R^{max}=150$. The number of selected nodes ranges from 5 to 25. As the number of selected nodes increases, the server utility generally increases under each budget constraint. However, the extent of the utility increase varies with the budget limit. For the lowest budget ($R^{max}=50$), the utility increases with the number of nodes but then levels off, indicating that beyond a certain point, additional nodes do not increase utility within the budget constraint. As the budget constraint increases from $R^{max}=100$ to $R^{max}=150$, utility continues to increase with the number of nodes, suggesting that higher budgets can effectively utilize more nodes to generate utility. The highest utility is observed under the constraint $R^{max}=150$, particularly with 25 nodes. This means that the server can utilize a higher budget to maximize its utility potential. The server's ability to generate utility is significantly affected by the number of nodes it can support, which in turn is limited by its budget.

Fig. \ref{fig6} illustrates the comparative analysis of server strategies under different conditions. The unit cost per node ranges from 1 to 5. Fig. \ref{fig6a} shows the server strategies under different unit satisfaction margins $\sigma$. As the node unit cost increases, the server strategies decrease for all values of $\sigma$. Higher satisfaction margins result in a sharp decrease in bids. This suggests that as the node cost increases, servers are inclined to bid higher for nodes with higher satisfaction margins. Fig. \ref{fig6b} shows the server's bid for different data sizes $d$. As in Fig. \ref{fig6a}, all three lines exhibit a decreasing trend, indicating that as node cost increases, the server's willingness to bid declines. This decrease is more significant at larger data volumes, implying that servers favor nodes with larger data sizes when bidding. Fig. \ref{fig6} shows that both unit satisfaction profit and data size significantly affect server strategy. As the node cost increases, servers tend to place lower bids, and higher satisfaction and larger data sizes amplify this effect.

\begin{figure}[t]
	\centerline{\includegraphics[width=0.5\textwidth]{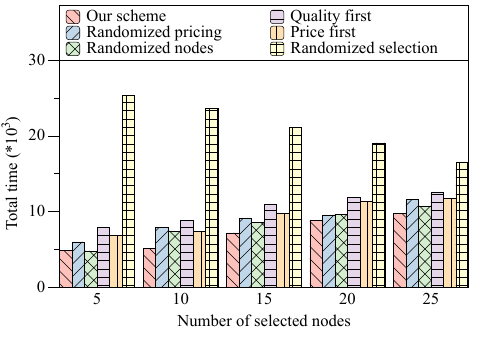}}
        \caption{FL training time for different number of selected nodes.}
	\label{fig8}
\end{figure}

\begin{figure}[!t]
	\centerline{\includegraphics[width=0.5\textwidth]{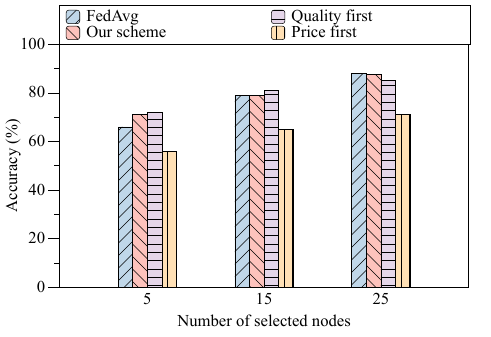}}
        \caption{FL accuracy for different number of selected nodes.}
	\label{fig9}
\end{figure}
\begin{figure}[!b]
	\centering
	\subfigure[]{
		\includegraphics[width=0.45\textwidth]{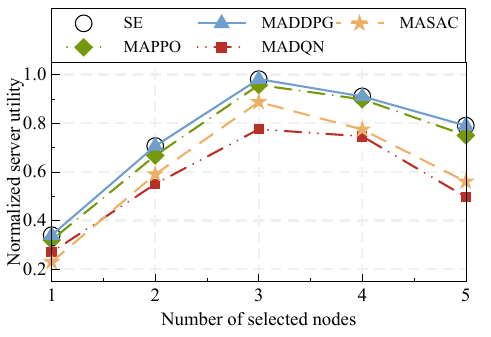}\label{fig15a}
	}
        \subfigure[]{
		\includegraphics[width=0.45\textwidth]{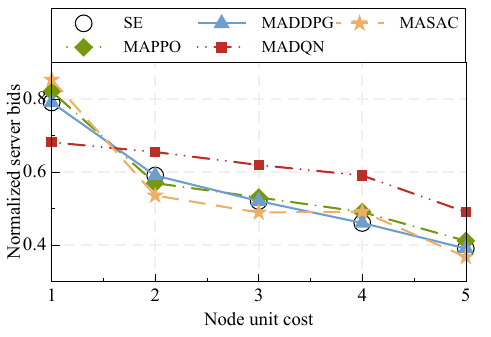}\label{fig15b}
	}
	\caption{Normalized server utility for different DRL algorithm and number of selected nodes.}
	\label{fig15}
\end{figure}

\subsection {FL Performance Comparison}
Fig. \ref{fig8} compares the total training time of different algorithms with varying total numbers of nodes. As shown in the figure, in the absence of incentives, the total training time of the global model decreases as the number of terminal devices increases. This is because the total computational resources increase as the number of devices increases. In contrast, in the incentive mechanism algorithm proposed in this paper, as the number of nodes increases, the budget available for each node decreases, which leads to a decrease in the total computational power, resulting in increased total training time. Notably, the randomized node combinations and the randomized selection algorithms show the most significant increase in training time as the node count grows, suggesting potential inefficiencies in these methods when dealing with larger sets of nodes. However, our scheme scales more efficiently, maintaining lower training times across all node numbers compared to the other strategies.

Fig. \ref{fig9} shows the accuracy comparison of FL under different schemes. The evaluation benchmarks the accuracy for 5, 15, and 25 nodes selected from a pool of 25 candidate nodes. The results indicate that our scheme achieves accuracy comparable to or better than FedAvg across all node sizes, suggesting that our scheme effectively manages the trade-offs between computation, communication, and model quality. The quality-first and price-first schemes show varying performance, with quality-first generally maintaining higher accuracy, suggesting that quality-first prioritizes model performance over cost. Conversely, price-first seems to prioritize cost-saving, potentially sacrificing accuracy.

\subsection{Solution Comparison}
Fig. \ref{fig15} compares the performance of different DRL algorithms in the Stackelberg game. In Fig. \ref{fig15a}, as the number of selected clients increases, the server utility exhibits a rise-then-fall trend, reaching its peak at 15 clients. This trend occurs because a moderate number of clients provides sufficient computational resources to optimize training performance.  However, as the number of clients continues to increase, the limited total budget results in less resource allocation per node, leading to a decline in overall utility. Of the four algorithms, MADDPG consistently achieves the highest utility, closely approximating the SE. In contrast, the utilities of MAPPO, MASAC, and MADQN are slightly lower, showing a noticeable gap. This demonstrates that MADDPG excels in balancing resource allocation and client selection to maximize server utility. In Fig. \ref{fig15b}, as the node unit cost increases, server strategies gradually decrease because the server must lower its bids to optimize resource allocation under budget constraints. The server strategies in MADDPG remain consistent with the SE, whereas the server strategies in MAPPO, MASAC, and MADQN diverge from the SE. Fig. \ref{fig15} shows that MADDPG consistently approximates the SE more closely than MAPPO, MASAC, and MADQN, indicating its superior performance in this Stackelberg game setting.

\section{Conclusion}
\label{secVIII}
In this paper, we have designed the MEFL framework for IIoT. The framework leverages meta-computing to enhance FL in IIoT. We have developed a satisfaction-aware incentive scheme within MEFL. This scheme integrates a satisfaction function that incorporates data size, AoI, and latency into the utility function, and formulates the utility optimization problem as a two-stage Stackelberg game. We employ a DRL-based approach to learn the equilibrium of the Stackelberg game, aiming to maximize the server's utility by identifying the optimal equilibrium. Experimental results demonstrate that, compared to traditional methods, the scheme effectively balances model quality and update latency through efficient resource allocation, enhances FL performance, ensures real-time capability and high efficiency of FL in IIoT, and better addresses the specific needs of industrial scenarios. However, the proposed MEFL still assumes rational and reliable participant behavior. Fully adversarial scenarios, such as intentional falsification of updates or coordinated manipulations, remain beyond the current scope. In future work, we plan to extend the framework by incorporating trust evaluation and anomaly detection mechanisms to further enhance its reliability and security in real-world IIoT environments.

\bibliographystyle{IEEEtran}
\bibliography{main}

\vspace{-30pt}

\begin{IEEEbiography}[{\includegraphics[width=1in,height=1.25in,clip,keepaspectratio]{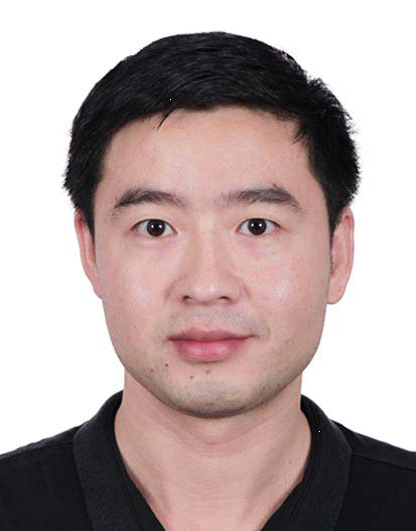}}]{Xiaohuan Li (Member, IEEE)}
was born in Chongqing, China. He received the B.Eng. and M.Sc. degrees from the Guilin University of Electronic Technology, Guilin, China, in 2006 and 2009, respectively, and the Ph.D. degree from the South China University of Technology, Guangzhou, China, in 2015. He was a Visiting Scholar with the Université de Nantes, France, in 2014. He is currently a Professor with the School of Information and Communication, Guilin University of Electronic Technology and Research fellow with the National Engineering Laboratory of Application Technology of Integrated Transportation Big Data (Beihang University). He is on the Editorial Boards of IEEE Internet of Things Journal and Journal on Communications. His current research interests include wireless sensor networks, vehicular networks, and cognitive radios.
\end{IEEEbiography}

\vspace{-20pt}

\begin{IEEEbiography}[{\includegraphics[width=1in,height=1.25in,clip,keepaspectratio]{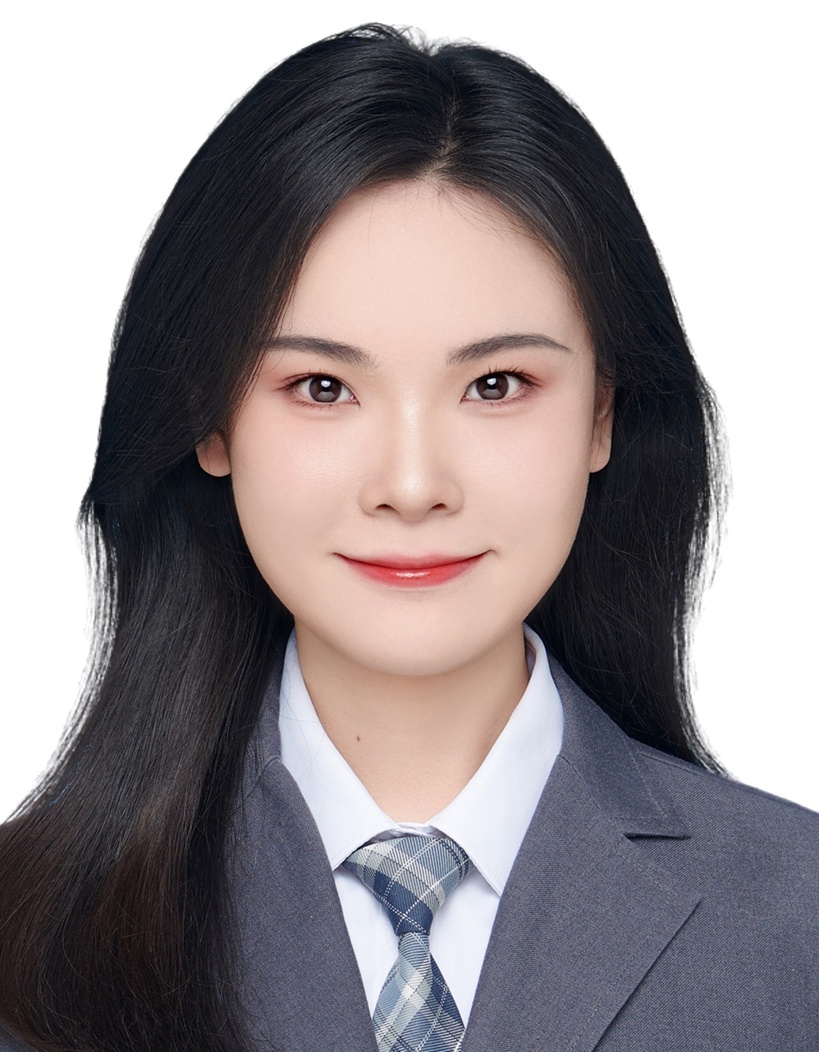}}]{Shaowen Qin}
received the B.S. degree from Guilin University of Electronic Technology, Guilin, China, in 2018, and the M.S. degree from Guangxi University, Nanning, China, in 2024. She is currently pursuing the Ph.D. degree in information and communication engineering at Guilin University of Electronic Technology. Her main research interests include federated learning and the industrial internet of things.
\end{IEEEbiography}


\begin{IEEEbiography}[{\includegraphics[width=1in,height=1.25in,clip,keepaspectratio]{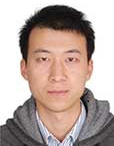}}]{Xin Tang}
received the B.S. and M.S. degrees in 2011 and 2015, respectively, from the Guilin University of Electronic Technology, Guilin, China, where he is currently working toward the Ph.D. degree in information and communication engineering. Since 2015, he has been with the China Mobile Communications Corporation Guangxi Branch. In 2016, he joined the Institute of Information Technology, Guilin University of Electronic Technology, where he is currently a Senior Engineer. He was engaged in a graduate study abroad program at Nanyang Technological University, Singapore, from July 2024 to July 2025. His research interests include edge computing, multi-agent systems, UAV networks, and intelligent transportation systems.
\end{IEEEbiography}

\vspace{-20pt}

\begin{IEEEbiography}[{\includegraphics[width=1in,height=1.25in,clip,keepaspectratio]{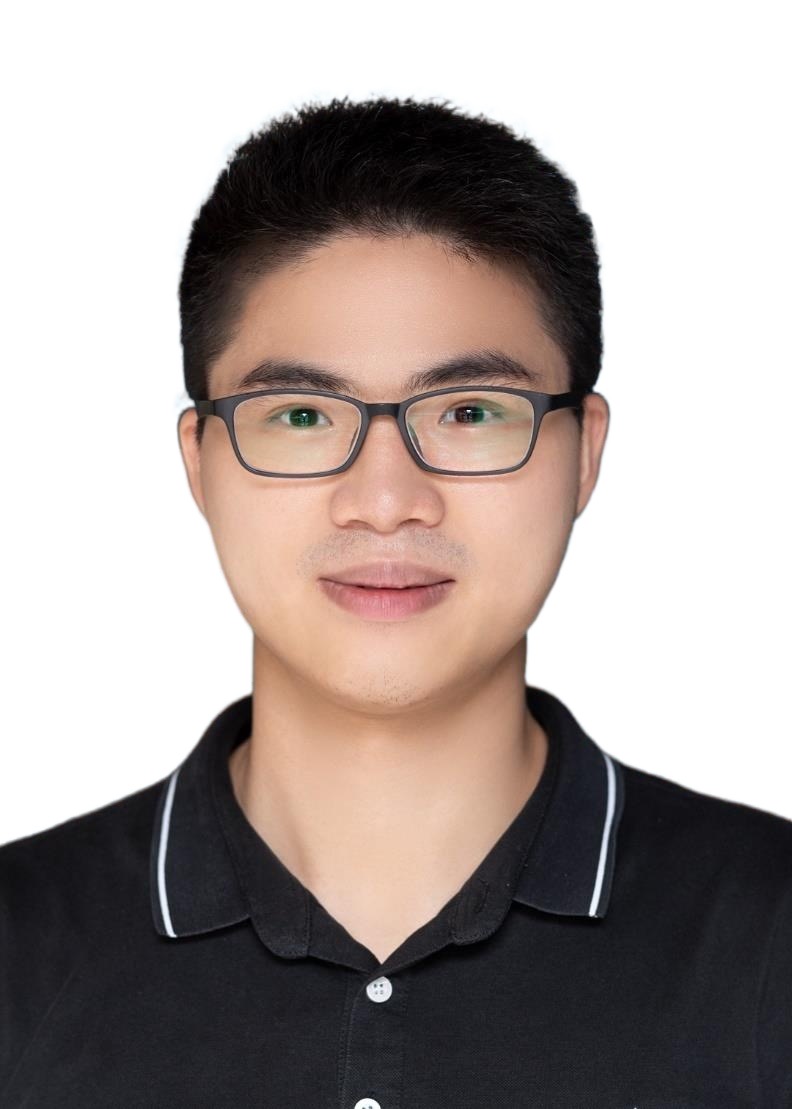}}]{Jiawen Kang}
received the Ph.D. degree from Guangdong University of Technology, China, in 2018. He has been a postdoc at Nanyang Technological University, Singapore from 2018 to 2021. He currently is a full professor at Guangdong University of Technology, China. His research interests mainly focus on blockchain, security, and privacy protection in wireless communications and networking.
\end{IEEEbiography}

\vspace{-20pt}

\begin{IEEEbiography}[{\includegraphics[width=1in,height=1.25in,clip,keepaspectratio]{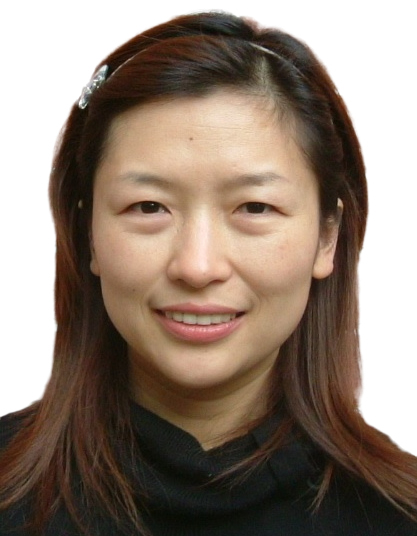}}]{Jin Ye}
received the Ph.D. degree with School of Science and Engineering, Central South University in 2008. She is currently a professor with school of Computer, Electronics and Information, Guangxi University. She worked as a visiting scholar in Department of Computer Science and Engineering at the University of Minnesota, Twin Cities in 2018. She also is the Member of China Computer Federation. Her main research interests include network protocol design, data center networks.
\end{IEEEbiography}

\vspace{-20pt}

\begin{IEEEbiography}[{\includegraphics[width=1in,height=1.25in,clip,keepaspectratio]{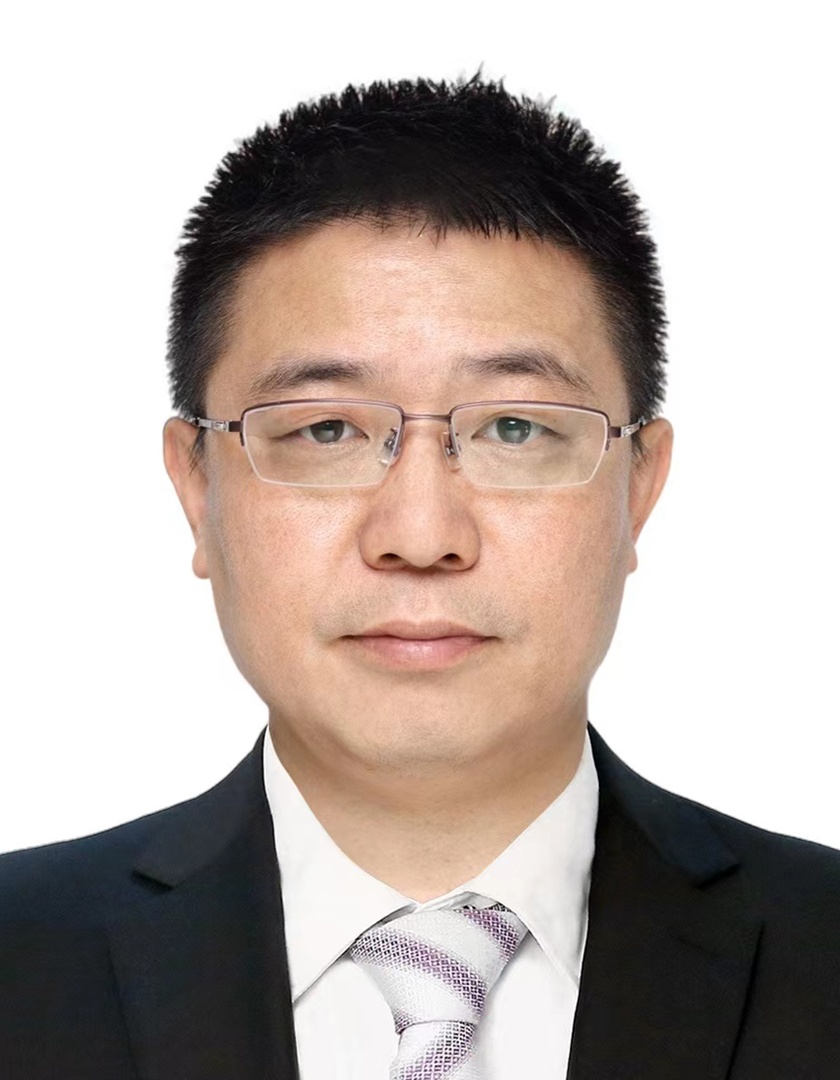}}]{Zhonghua Zhao}
received the M.S. degree from Guilin University of Electronic Technology, Guilin, China. He is currently a professor with the School of Information and Communication, Guilin University of Electronic Technology, Guilin, China, and has published 18 academic papers. His current research interests include vehicular networks, edge computing, and intelligent transportation systems.
\end{IEEEbiography}

\vspace{-20pt}

\begin{IEEEbiography}[{\includegraphics[width=1in,height=1.25in,clip,keepaspectratio]{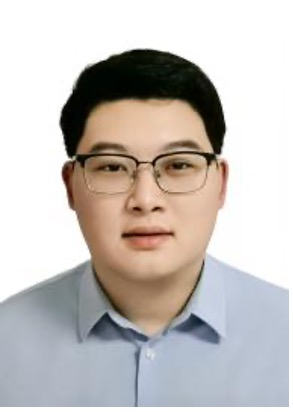}}]{Yusi Zheng}
received the B.S. degree from Guilin University of Electronic Technology, Guilin, China, in 2024, and the M.S. degree from Lingnan University, Hong Kong, China, in 2025.  He is currently pursuing the Ph.D. degree in information and communication engineering with the Guilin University of Electronic Technology. His research focuses on intelligent transportation systems, with applications in highway traffic event detection and UAV-based solutions.
\end{IEEEbiography}

\vspace{-20pt}

\begin{IEEEbiography}[{\includegraphics[width=1in,height=1.25in,clip,keepaspectratio]{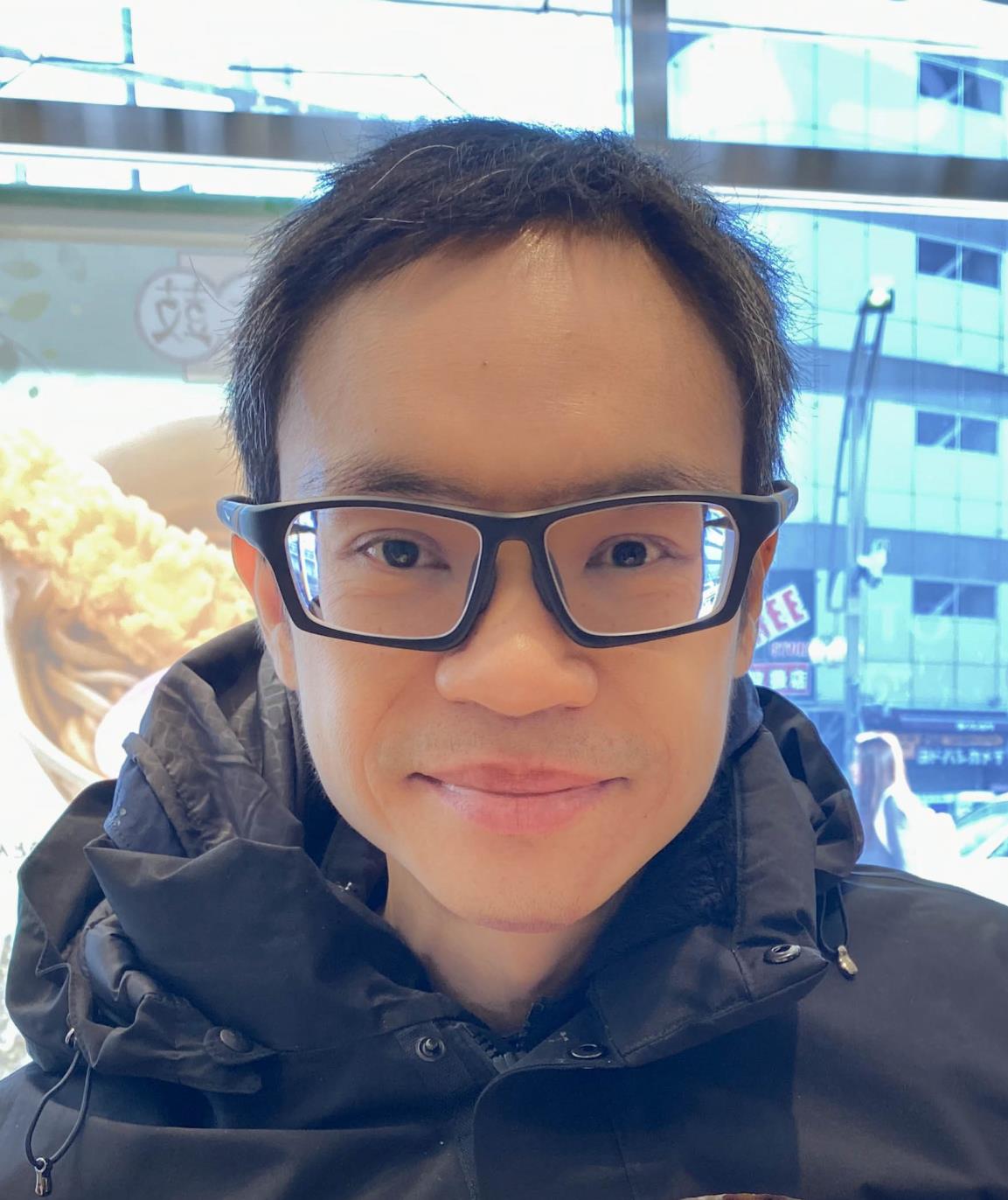}}]{Dusit Niyato (Fellow, IEEE)}
is a professor in the School of Computer Science and Engineering, at Nanyang Technological University, Singapore. He received B.Eng. from King Mongkuts Institute of Technology Ladkrabang (KMITL), Thailand in 1999 and Ph.D. in Electrical and Computer Engineering from the University of Manitoba, Canada in 2008. His research interests are in the areas of sustainability, edge intelligence, decentralized machine learning, and incentive mechanism design. 
\end{IEEEbiography}

\end{document}